\documentclass{article}
\usepackage{ijcai19}

\usepackage{times}
\usepackage{soul}
\usepackage{url}
\usepackage[hidelinks]{hyperref}
\usepackage[utf8]{inputenc}
\usepackage[small]{caption}
\usepackage{graphicx}
\usepackage{amsmath}
\usepackage{booktabs}
\usepackage{algorithm}
\usepackage{algorithmic}
\usepackage{natbib}

\usepackage[dvipsnames]{xcolor}
\usepackage{hyperref}
\usepackage{xspace}
\usepackage{amsfonts}
\usepackage{amsthm}
\usepackage{multirow}
\usepackage{subcaption}
\usepackage{xifthen}
\usepackage{tikz}

\usetikzlibrary{arrows,shapes,backgrounds}
\usetikzlibrary{decorations.markings}
\usetikzlibrary{calc}
\usetikzlibrary{positioning}

\tikzstyle{bnarrow}=[
    decoration={markings,mark=at position 1 with {\arrow[scale=1.5]{>}}},
    postaction={decorate},
    shorten >=0.7pt,
    >=latex,
    line width=0.3
]
\tikzstyle{bayesnet}=[
  >=latex, thick, auto
]
\tikzstyle{bnnode}=[
  draw,ellipse,minimum size=7mm,inner sep=1pt,font=\small
]
\tikzstyle{cpt}=[
  font=\scriptsize
]

\makeatletter

\let\c@table\c@figure
\makeatother

\DeclareMathOperator*{\Ex}{\mathbb{E}}
\DeclareMathOperator*{\argmax}{arg\,max}
\DeclareMathOperator*{\argmin}{arg\,min}

\DeclareMathOperator{\sgn}{sgn}
\DeclareMathOperator{\sigmoid}{sigmoid}

\newcommand{\NP}{\ensuremath{\mathsf{NP}}}
\newcommand{\SharpP}{\ensuremath{\mathsf{\#P}}}

\newcommand\given[1][]{\,#1\vert\,}

\newcommand{\Real}{\ensuremath{\mathbb{R}}\xspace}

\newcommand{\rvars}[1]{\ensuremath{\mathbf{#1}}\xspace}
\newcommand{\Xs}{\rvars{X}}
\newcommand{\Ys}{\rvars{Y}}

\newcommand{\Ms}{\rvars{M}}

\newcommand{\jstate}[1]{\ensuremath{\mathbf{#1}}\xspace}
\newcommand{\xs}{\jstate{x}}
\newcommand{\ys}{\jstate{y}}

\newcommand{\es}{\jstate{e}}

\newcommand{\thetas}{\jstate{\theta}}
\newcommand{\ms}{\mathbf{m}\xspace}

\newcommand{\ws}{\mathbf{w}\xspace}

\newcommand{\cs}{c}

\newcommand{\function}{\ensuremath{\mathcal{F}}\xspace}
\newcommand{\params}{\ensuremath{\mathcal{P}}\xspace}

\newcommand{\pa}[2]{%
  \ifthenelse{\isempty{#2}}%
    {\ensuremath{\theta_{#1}}\xspace}
    {\ensuremath{\theta_{#1 \vert #2}}\xspace}
}
\newcommand{\notc}{\ensuremath{\bar{c}}\xspace}

\newcommand{\abs}[1]{\left\lvert#1\right\rvert}

\theoremstyle{plain} 
\newtheorem{theorem}{Theorem}
\newtheorem{lemma}{Lemma}
\newtheorem{proposition}{Proposition}

\theoremstyle{definition}
\newtheorem{definition}{Definition}
\newtheorem{example}{Example}

\title{What to Expect of Classifiers? \\ Reasoning about Logistic Regression with Missing Features}

\author{
    Pasha Khosravi
    \and
    Yitao Liang\and
    YooJung Choi\And
    Guy Van den Broeck
    \affiliations
    Computer Science Department \\
    University of California, Los Angeles\\
    \emails
    \{pashak,\ yliang,\ yjchoi,\ guyvdb\}@cs.ucla.edu
}

\begin{document}

\maketitle

\begin{abstract}
    While discriminative classifiers often yield strong predictive performance, missing feature values at prediction time can still be a challenge. Classifiers may not behave as expected under certain ways of substituting the missing values, since they inherently make assumptions about the data distribution they were trained on. In this paper, we propose a novel framework that classifies examples with missing features by computing the expected prediction with respect to a feature distribution. Moreover, we use geometric programming to learn a naive Bayes distribution that embeds a given logistic regression classifier and can efficiently take its expected predictions. Empirical evaluations show that our model achieves the same performance as the logistic regression with all features observed, and outperforms standard imputation techniques when features go missing during prediction time. Furthermore, we demonstrate that our method can be used to generate ``sufficient explanations'' of logistic regression classifications, by removing features that do not affect the classification.
\end{abstract}

\section{Introduction}

Missing values are pervasive in real-world machine learning applications.
Learned classifiers usually assume all input features are known, but when a classifier is deployed, some features may not be available, or too difficult to acquire. This can be due to the noisy nature of our environment, unreliable or costly sensors, and numerous other challenges in gathering and managing data \citep{dekel2008learning,graham2012missing}. Consider autonomous driving for example, where blocked sensors may leave observations incomplete.

Moreover, it can be difficult to anticipate the many ways in which a learned classifier will be deployed. For example, the same classifier could be used by different doctors who choose to run different medical tests on their patients, and this information is not known at learning time.
Nevertheless, even a small portion of missing data can severely affect the performance of well-trained models, leading to predictions that are very different from the ones made when all data is~observed. 

Certain machine learning models, such as probabilistic graphical models, provide a natural solution to missing features at prediction time, by formulating the problem as a probabilistic inference task~\citep{koller2009probabilistic,darwiche2009modeling}. 
Unfortunately, with the increasing emphasis on predictive performance, the general consensus is that
such generative models are not competitive as classifiers given fully-observed feature vectors, and that discriminatively-trained models are to be preferred~\citep{ng2002discriminative}.

To alleviate the impact of missing data on discriminative classifiers, it is common practice to substitute the missing values with plausible ones \citep{schafer1999multiple,little2014statistical}. As we will argue later, a drawback for such imputation techniques is that they can make overly strong assumptions about the feature distribution. Furthermore, to be compatible with many types of classifiers, they tend to overlook how the multitude of possible imputed values would interact with the classifier at hand, and risk yielding biased predictions.

To better address this issue, we propose a principled framework of handling missing features by reasoning about the classifier's \emph{expected} output given the feature distribution. One obvious advantage is that it can be tailored to the given family of classifiers and feature distributions.  In contrast, the popular mean imputation approach only coincides with the expected prediction for simple models (e.g. linear functions) under very strong independence assumptions about the feature distribution.
We later show that calculating the expected predictions with respect to arbitrary feature distributions is computationally highly intractable. In order to make our framework more feasible in practice, we leverage generative-discriminative counterpart relationships to learn a joint distribution that can take expectations of its corresponding discriminative classifier. We call this problem \emph{conformant learning}.
Then, we develop an algorithm, based on geometric programming, for a well-known example of such relationship: naive Bayes and logistic regression. We call this specific algorithm naive conformant learning (NaCL).

Through an extensive empirical evaluation over five characteristically distinct datasets, we show that NaCL consistently achieves better estimates of the conditional probability, as measured by average cross entropy and classification accuracy, compared to commonly used imputation methods. Lastly, we conduct a short case study on how our framework can be applied to \emph{explain} classifications.

\section{The Expected Prediction Task}
In this section, we describe our intuitive approach to making a prediction when features are missing and discuss how it relates to existing imputation methods. Then we study the computational hardness of our expected prediction task.

We use uppercase letters to denote features/variables and lowercase letters for their assignment. Sets of variables $\Xs$ and their joint assignments $\xs$ are written in bold.
For an assignment $x$ to a binary variable $X$, we let $\bar{x}$ denote its negation.
Concatenation $\Xs\Ys$ denotes the union of disjoint sets. The set of all possible assignments to $\Xs$ is denoted~$\mathcal{X}$.

Suppose we have a model trained with features $\Xs$ but are now faced with the challenge of making a prediction without knowing all values $\xs$. 
In this situation, a common solution is to impute certain substitute values for the missing data (for example their mean)~\citep{little2014statistical}.
However, the features that were observed provide information not only about the class but also about the missing features, yet this information is typically not taken into account by popular methods such as mean imputation.

We propose a very natural alternative: to utilize the feature distribution to probabilistically reason about what a predictor is expected to return if it could observe the missing features.
\begin{definition}
    Let $\function: \mathcal{X} \to \Real$ be a predictor
    and $P$ be a distribution over features $\Xs$. Given a partitioning of features $\Xs = \Ys\Ms$ and an assignment $\ys$ to some of the features $\Ys$, the \emph{expected prediction task} is to compute
    \begin{equation*}
        E_{\function, P}(\ys) = \Ex_{\ms \sim P(\Ms \mid \ys)} \left[ \function(\ys \ms) \right].
    \end{equation*}
\end{definition}

The expected prediction task and (mean) imputation are related, but only under very restrictive assumptions.
\begin{example}\label{ex:linear}
    Let $\function : \mathcal{X} \to \Real$ be a linear function. That is, $\function(\xs) = \sum_{x \in \xs} w_X x$ for some weights $\ws$. Suppose $P$ is a distribution over $\Xs$ that assumes independence between features: $P(\Xs) = \prod_{X\in\Xs} P(X)$. Then, using linearity of expectation, the following holds for any partial observation~$\ys$:
    \begin{align*} 
        E_{\function, P}(\ys) &= \Ex_{\ms \sim P(\Ms \mid \ys)} \left[ \sum_{y \in\ys} w_Y y + \sum_{m \in\ms} w_M m \right] \\
        &= \sum_{y \in\ys} w_Y y + \sum_{M \in \Ms} w_M \Ex_{m \sim P(M)}[m].
    \end{align*}
    Hence, substituting the missing features with their means effectively computes the expected predictions of linear models if the independence assumption holds. Furthermore, if $\function$ is the true conditional probability of the labels and features are generated by a fully-factorized $P(\Xs)$, then classifying by comparing the expected prediction $E_{\function,P}$ to 0.5 is Bayes optimal on the observed features. That is, an expected prediction higher than 0.5 means that the positive class is the most likely one given the observation, and thus minimizes expected loss, according to the distribution defined by $\function$ and $P$.
\end{example}

\begin{example}
    Consider a logistic regression model $\mathcal{G}(\xs) = \sigmoid(\function(\xs))$ where $\function$ is a linear function. 
    Now, mean imputation no longer computes the expected prediction, even when the independence assumption in the previous example holds.
    In particular, if $\ys$ is a partial observation such that $\mathcal{G}(\ys\ms)$ is positive for all $\ms$, then the mean-imputed prediction is an over-approximation of the expected prediction:
    \begin{align*}
        &\mathcal{G}\!\left(\ys \Ex[\ms]\right) = \sigmoid\left(\Ex[\function(\ys\ms)]\right) \\
        & \qquad > \Ex[\sigmoid(\function(\ys\ms))] = E_{\mathcal{G},P}(\ys).
    \end{align*}
    This is due to Jensen's inequality and concavity of the sigmoid function in the positive portion; conversely, it is an under-approximation in the negative cases.
\end{example}

Example~\ref{ex:linear} showed how to efficiently take the expectation of a linear function w.r.t.\ a fully factorized distribution. Unfortunately, the expected prediction task is in general computationally hard, even on simple classifiers and distributions.

\begin{proposition} \label{prop:uniform-cnf-expectation}
Taking expectation of a nontrivial classifier w.r.t.\ a uniform distribution is \SharpP-hard.
\end{proposition}
\begin{proof}
    Suppose our classifier tests whether a logical constraint holds between the input features. Then asking whether there exists a positive example is equivalent to SAT which is \NP-hard. The expected classification on a uniform distribution is solving an even harder task, of counting solutions to the constraint, which is \SharpP-hard~\citep{roth1996hardness}.
\end{proof}
Next, consider the converse in which the classifier is trivial but the distribution is more general.
\begin{proposition}
The expectation of a classifier that returns the value of a single feature w.r.t.\ a distribution represented by a probabilistic graphical model is \SharpP-hard.
\end{proposition}
\begin{proof}
    Computing expectations of such classifier is as hard as computing marginals in the feature distribution, which is \SharpP-hard for graphical models~\citep{roth1996hardness}.
\end{proof}
Previous propositions showed that the expected prediction task stays intractable, even when we allow either the distribution or the classifier to be trivial.

Our next theorem states that the task is hard even for a relatively simple classifier 
and a tractable distribution.\footnote{All proofs can be found in Appendix~\ref{sec:appx-proofs}.
}
\begin{theorem}\label{thm:lr-over-nb}
    Computing the expectation of a logistic regression classifier over a naive Bayes distribution is \NP-hard.
\end{theorem}
That is, the expected prediction task is hard even though logistic regression classification and probabilistic reasoning on naive Bayes models can both be done in linear time. 

In summary, while the expected prediction task appears natural for dealing with missing data, its vast intractability poses a serious challenge, especially compared to efficient alternatives such as imputation.
Next, we investigate specific ways of practically overcoming this challenge.

\section{Joint Distributions as Classifiers} \label{section:GMasC}

As shown previously, taking expectations is intractable for arbitrary pairs of classifiers and distributions. In this section, we propose \emph{conformant learning} which aims to learn a joint distribution that encodes a given classifier as well as a feature distribution. On such distribution, the expected prediction task is well-defined as probabilistic inference, and is tractable for a large class of probabilistic models, including naive Bayes~\citep{darwiche2009modeling,dechter2013reasoning}.

We first describe how a joint distribution can be used as a classifier that inherently support missing features during prediction time.
Given a distribution $P(\Xs,C)$ over the features $\Xs$ and class variable $C$, we can classify a partial observation $\ys$ simply by computing the conditional probability $P(c \given \ys)$ where $c$ denotes the positive class.\footnote{We assume binary class for conciseness, but our approach easily generalizes to multiclass. Details can be found in Appendix~\ref{sec:appx-multiclass}.} In some sense, a joint distribution embeds a classifier $P(C \given \Ys)$ for each subset of observed features $\Ys$. 
In fact, computing $P(c \given \ys)$ is \emph{equivalent to computing the expected prediction} of classifier $\function$ that outputs $P(c \given \xs)$ for every $\xs$, with respect to distribution $P(\Xs)$:
\begin{align}
   P(c \given \ys) 
     & = \sum_\ms P(c, \ms \given \ys)
       = \sum_\ms P(c \given \ms\ys) P(\ms \given \ys) \nonumber\\
     & = \Ex_{\ms \sim P(\Ms \mid \ys)} \left[ P(c \given \ys\ms) \right] 
       = E_{\function,P}(\ys). \label{eq:jointexpectation}
\end{align}
Nevertheless, the prevailing consensus is that in practice discriminatively training a classifier $P(C \given \Xs)$ should be preferred to generatively learning $P(\Xs,C)$, because it tends to achieve higher classification accuracy \citep{bouchard2004tradeoff,ulusoy2005generative}. 

There are many generative-discriminative pairs obtained from fitting the same family of probabilistic models to optimize either the joint or conditional likelihood~\citep{jaakkola1999exploiting,crftut:fnt,liang2019logistic}, including naive Bayes and logistic regression~\citep{ng2002discriminative}. We formally describe such relationship as follows:
\begin{definition}
    We say $P(\Xs,C)$ \emph{conforms} with ${\function: \mathcal{X} \to [0,1]}$ if their classifications agree: $P(c \given \xs) = \function(\xs)$ for all $\xs$.
\end{definition}

Next, let us study naive Bayes models in detail as they support efficient inference and thus are a good target distribution to leverage for the expected prediction task.
Naive Bayes models assume that features are mutually independent given the class; that is, its joint distribution is $P(\Xs, C) = P(C) \prod_{X \in \Xs} P(X \given C)$.
Under such assumption, marginal inference takes linear time, and so does computing expectations under missing features as in Equation~\ref{eq:jointexpectation}.

Logistic regression is the discriminative counterpart to naive Bayes. It has parameters $\ws$ and posits that
\footnote{Here, $\xs$ also includes a dummy feature that is always $1$ to correspond with the bias parameter $w_0$.} 
\begin{equation*}
    \function(\xs ) = \frac{1}{1+e^{-\ws^T \cdot \xs}}.
\end{equation*}
Any naive Bayes classifier can be translated to an equivalent logistic regression classifier on fully observed features.
\begin{lemma} \label{lem:nb-to-lr}
Given a naive Bayes distribution $P$, there is a unique logistic regression model $\function$ that it conforms with.
Such logistic regression model has the following weights:
\begin{align*}
    w_0 &= \log \cfrac{P(c)}{P(\notc)} + \sum_{i=1}^{n} \log \cfrac{ P(\bar{x}_i \mid c)}{ P(\bar{x}_i \mid \notc)} \\
    w_i &= \log \cfrac{ P(x_i \mid c)}{ P(x_i \mid \notc)} \cdot \cfrac{ P(\bar{x}_i \mid \notc)}{ P(\bar{x}_i \mid c)}, \quad i=1,\dots,n
\end{align*}
Here, $x$, $\bar{x}$ denote $X\!=\!1$, $X\!=\!0$ respectively.
\end{lemma}
The lemma above can be drawn from~\citet{Roos2005,ng2002discriminative}. Complete proofs of all lemmas can be found in Appendix~\ref{sec:appx-proofs}.

\begin{figure}[tb]
    \centering
    \begin{subfigure}[c]{\columnwidth}
        \centering
        \begin{tikzpicture}
            \node (func) at (0bp,0bp) {
                \footnotesize
                $
                \ws = \begin{bmatrix}-1.16\\2.23\\-0.20\end{bmatrix}
                $
            };
            
            \node[right=5bp of func] {
                \scriptsize
                \begin{tabular}{ccc}
                    \toprule
                    $X_1$ & $X_2$ & $\function(x_1,x_2)$\\
                    \midrule
                    1 & 1 & 0.70 \\
                    1 & 0 & 0.74 \\
                    0 & 1 & 0.20 \\
                    0 & 0 & 0.24 \\
                    \bottomrule
                \end{tabular}
            };
        \end{tikzpicture}
        \caption{Logistic regression weights and resulting predictions.}\label{fig:ex-lr}
    \end{subfigure}
    \\
    \begin{subfigure}[c]{\columnwidth}
        \centering
      \begin{tikzpicture}[bayesnet]
          \def\lone{15bp}
          \def\ltwo{-15bp}
        
          \node (C) at (0bp,\lone) [bnnode] {$C$};
          \node (X1) at (-15bp,\ltwo) [bnnode] {$X_1$};
          \node (X2) at (15bp,\ltwo) [bnnode] {$X_2$};
          
          \begin{scope}[on background layer]
            \draw [bnarrow] (C) -- (X1);
            \draw [bnarrow] (C) -- (X2);
          \end{scope}
          
          \node[cpt] (cpt1-c) at (50bp,22bp) {
            \begin{tabular}{c}
              \toprule
              $P_1(c)$\\\midrule
              \multirow{2}{*}{$0.5$}\\
              \\
              \bottomrule
            \end{tabular}
          };
          
          \node[cpt,right=0bp of cpt1-c] (cpt1-x1){
            \begin{tabular}{ l c}
              \toprule
              $C$ & $P_1(x_1|C)$\\\midrule
              1 & $0.8$\\
              0 & $0.3$\\
              \bottomrule
            \end{tabular}
          };
          
          \node[cpt,right=0bp of cpt1-x1]{
            \begin{tabular}{lc}
              \toprule
              $C$ & $P_1(x_2|C)$\\\midrule
              1 & $0.45$\\
              0 & $0.5$\\
              \bottomrule
            \end{tabular}
          };
          
          \node[cpt] (cpt2-c) at (50bp,-22bp) {
            \begin{tabular}{c}
              \toprule
              $P_2(c)$\\\midrule
              \multirow{2}{*}{$0.36$}\\
              \\
              \bottomrule
            \end{tabular}
          };
          
          \node[cpt,right=0bp of cpt2-c] (cpt2-x1){
            \begin{tabular}{lc}
              \toprule
              $C$ & $P_2(x_1|C)$\\\midrule
              1 & $0.6$\\
              0 & $0.14$\\
              \bottomrule
            \end{tabular}
          };
          
          \node[cpt,right=0bp of cpt2-x1]{
            \begin{tabular}{lc}
              \toprule
              $C$ & $P_2(x_2|C)$\\\midrule
              1 & $0.9$\\
              0 & $0.92$\\
              \bottomrule
            \end{tabular}
          };
    \end{tikzpicture}
    \caption{Two naive Bayes distributions with same structure.}\label{fig:ex-nb}
    \end{subfigure}
    \caption{Logistic regression $\function(\xs)=\sigmoid(\ws^T \xs)$ and two conformant naive Bayes models.}
\end{figure}

Consider for example the naive Bayes (NB) distribution $P_1$ in Figure~\ref{fig:ex-nb}. For all possible feature observations, the NB classification $P_1(c\given\xs)$ is equal to that of logistic regression (LR) $\function$ in Figure~\ref{fig:ex-lr}, whose weights are as given by above lemma (i.e., $P_1$ conforms with $\function$).
Furthermore, distribution $P_2$ also translates into the same logistic regression. In fact, there can be infinitely many such naive Bayes distributions.
\begin{lemma}
\label{lem:lr-to-nb}
Given a logistic regression $\function$ and ${\thetas \in (0,1)^n}$, there exists a unique naive Bayes model $P$ such that
\begin{align*}
    &P (c \given \xs) = \function(\xs), \quad \forall\: \xs \\
    &P (x_i \given \cs) = \theta_i, \quad i=1,\dots,n.
\end{align*}
\end{lemma}
That is, given a logistic regression there are infinitely many naive Bayes models that conform with it. Moreover, after fixing values for $n$ parameters of the NB model, there is a uniquely corresponding naive Bayes model.

We can expect this phenomenon to generalize to other generative-discriminative pairs; given a conditional distribution $P(C\given\Xs)$ there are many possible feature distributions $P(\Xs)$ to define a joint distribution $P(\Xs,C)$. 
For instance, distributions $P_1$ and $P_2$ in Figure~\ref{fig:ex-nb} assign different probabilities to feature observations; $P_1(\bar{x}_1,\bar{x}_2)=0.23$ whereas $P_2(\bar{x}_1,\bar{x}_2)=0.06$.
Hence, we wish to define which one of these models is the ``best''. Naturally, we choose the one that best explains a given dataset of feature observations.\footnote{Here we assume i.i.d.\ sampled data. If a true distribution is known, we can equivalently minimize the KL-divergence to it.}

\begin{definition}
    Let $\function: \mathcal{X} \to [0,1]$ be a discriminative classifier and $D$ be a dataset where each example $d$ is a joint assignment to $\Xs$. Given a family of distributions over $C$ and $\Xs$, let $\params$ denote the subset of them that conforms with $\function$.
    Then \emph{conformant learning} on $D$ is to solve the following optimization:
    \begin{equation}
        \argmax_{P \in \params} \prod_{d \in D} P(d) = \argmax_{P \in \params} \prod_{d=(\xs) \in D} \sum_c P(\xs,c). \label{eq:conformant}
    \end{equation}
\end{definition}
The learned model thus conforms with $\function$ and defines a feature distribution; therefore, we can take the expectation of $\function$ via probabilistic inference. In other words, it attains the desired classification performance of the given discriminative model while also returning sophisticated predictions under missing features.
Specifically, conformant  naive Bayes models can be used to efficiently take expectations of logistic regression classifiers.
Note that this does not contradict Theorem~\ref{thm:lr-over-nb} which considers arbitrary pairs of LR and NB models.


\section{Naive Conformant Learning}
\label{sec:logistic_regression_expectations}
In this section, we study a special case of conformant learning -- naive conformant learning (NaCL), and show how it can be solved as a geometric program.

A naive Bayes distribution is defined by a parameter set $\theta$ that consists of $\pa{c}{}, \pa{\notc}{}$, and $\pa{x}{c}, \pa{x}{\notc}$ for all $x$.
\emph{Naive conformant learning} outputs the naive Bayes distribution $P_\theta$ that maximizes the (marginal) likelihood given the dataset and conforms with a given logistic regression model $\function$.

We will next show that above problem can be formulated as a \emph{geometric program}, an optimization problem of the form:
\begin{alignat*}{3}
    &\text{min}\: &&f_0(x) &&\\
    &\text{s.t} && f_i(x) \leq 1, && \quad i=1 \ldots m \\
    & && g_i(x) = 1, && \quad i=1 \ldots p
\end{alignat*}
where each $f_i$ is a posynomial and $g_i$ monomial. A \emph{monomial} is a function of the form $b x_1^{a_1}\cdots x_n^{a_n}$ defined over positive real variables $x_1,\dots,x_n$ where $b>0$ and $a_i\in\Real$. A \emph{posynomial} is a sum of monomials.
Every geometric program can be transformed into an equivalent convex program through change of variables, and thus its global optimum can be found efficiently~\citep{boyd2007tutorial}.

To maximize the likelihood, we instead minimize its inverse. 
Let $n(\xs)$ denote the number of times an assignment $\xs$ appears in dataset $D$. Then the objective function is:
\begin{equation*}
\resizebox{.91\linewidth}{!}{$
    \displaystyle
    \prod_{d\in D} P_\theta(d)^{-1}
    \! = \prod_\xs P_\theta(\xs)^{-n(\xs)}
    \! = \prod_\xs {\left(\sum_c \prod_{x\in\xs}\pa{x}{c} \pa{c}{}\right)}^{\! \! \! -n(\xs)}\! \! .
$}
\end{equation*}
Above formula, directly expanded, is not a posynomial. In order to express it as a posynomial we consider an auxiliary dataset $D^\prime$ constructed from $D$ as follows: 
for each data point $d_j\!=\!(\xs) \in D$, there are $d_{j,c}^\prime\!=\!(\xs,c) \in D^\prime$ with weight $\alpha_{j,c}$ for all values of $c$.
If the weights are such that $\alpha_{j,c} \geq 0$ and $\sum_c \alpha_{j,c}=1$ for all $j$, then the inverse of the expected joint likelihood given the new dataset $D^\prime$ is
\begin{align}
    &\prod_{d_{j,c}^\prime=(\xs,c) \in D^\prime} P_\theta(\xs,c)^{-\alpha_{j,c}} \nonumber \\
    &= \prod_{d_j=(\xs)\in D} \prod_c P_\theta(\xs)^{-\alpha_{j,c}} P_\theta(c \given \xs)^{-\alpha_{j,c}} \nonumber \\
    &= \prod_{d\in D} P_\theta(d)^{-1} \cdot \prod_{d_j=(\xs)\in D, c} P_\theta(c\given\xs)^{-\alpha_{j,c}}. \label{eq:decomp-ll}
\end{align}
For any $P_\theta \in \params$, the conditional distribution $P_\theta(C \given \Xs)$ is fixed by the logistic regression model; in other words, the last product term in Equation~\ref{eq:decomp-ll} is a constant. Therefore, maximizing the expected (joint) likelihood on a completed dataset must also maximize the marginal likelihood, which is our original objective.
Intuitively, maximizing the joint likelihood on any dataset is equivalent to maximizing the marginal likelihood $P(\Xs)$ if the conditional distribution $P(C \given \Xs)$ is fixed.
Now our objective function can be written as a monomial in terms of the parameters:
\begin{equation}
\resizebox{.91\linewidth}{!}{$
    \displaystyle
    \prod_{d_{j,c}^\prime \in D^\prime} P_\theta(d_{j,c}^\prime)^{-\alpha_{j,c}}
    = \prod_{d_{j,c}^\prime=(\xs,c) \in D^\prime} \left( \pa{c}{} \prod_{x\in\xs} \pa{x}{c} \right)^{-\alpha_{j,c}}.
$} \label{eq:joint-ll-monomial}
\end{equation}

Next, we express the set of conformant naive Bayes distributions $\params$ as geometric program constraints in terms of $\theta$.
An NB model $P_\theta$ conforms with an LR $\function$ if and only if its corresponding logistic regression weights, according to Lemma~\ref{lem:nb-to-lr}, match those of $\function$.
Hence, $P_\theta \in \params$ precisely when
\begin{gather}
    e^{w_0}\ \pa{c}{}^{-1}\ \pa{\notc}{}\ \prod_{i=1}^n\  \pa{\bar{x}_i}{c}^{-1}\ \pa{\bar{x}_i}{\notc} = 1 \label{eq:constraint_w0} \\
    e^{w_i}\ \pa{x_i}{c}^{-1}\ \pa{x_i}{\notc}\ \pa{\bar{x}_i}{c}\ \pa{\bar{x}_i}{\notc}^{-1} = 1, \ \forall\:i \label{eq:constraint_wi}
\end{gather}

We also need to ensure that the parameters define a valid probability distribution (e.g., \pa{c}{} + \pa{\notc}{} = 1). Because such posynomial equalities are not valid geometric program constraints, we instead relax them to posynomial inequalities:\footnote{The learned parameters may not sum to 1. They can still be interpreted as a multi-valued NB with same likelihood that conforms with $\function$. These constraints were always active in our experiments.}
\begin{gather}
\resizebox{.91\linewidth}{!}{$
    \pa{c}{} + \pa{\notc}{} \leq 1,\: 
    \pa{x_i}{c} + \pa{\bar{x}_i}{c} \leq 1,\: \pa{x_i}{\notc} + \pa{\bar{x}_i}{\notc} \leq 1, \:\forall\:i
$}\label{eq:sum-to-one}
\end{gather}
Putting everything together, naive conformant learning can be solved as a geometric program whose objective function is given by Equation~\ref{eq:joint-ll-monomial} and constraints by Equations~\ref{eq:constraint_w0}~--~\ref{eq:sum-to-one}.
We used the GPkit library \citep{gpkit} to solve our geometric programs.


\begin{table}[t]
    \centering
        {\fontsize{6.5}{9}\selectfont
    \begin{sc}
    \begin{tabular}{@{}l c c  c  c @{}}
        \toprule
        Datasets & Size & \# Classes: Dist. & \# Features & Feature Types \\
        \midrule 
        MNIST & 60K & 10: Balanced & 784 & Int.\ pixel value \\
        Fashion & 60K & 10: Balanced & 784 & Int.\ pixel value \\
        CovType & 581K & 7: Unbalanced & 54 & Cont.\ $\&$ Categorical \\
        Adult & 49K & 2: Unbalanced & 14 & Int.\ $\&$ Categorical \\
        Splice & 3K & 3: Unbalanced & 61 & Categorical \\
        \bottomrule
    \end{tabular}
    \end{sc}
    }
    \caption{Summary of our testbed.}
    \label{table: datasets}
\end{table}
\begin{figure*}[tb]
    \centering
        \begin{subfigure}[c]{.49\textwidth}
            \centering
            \includegraphics[height=0.28\columnwidth]{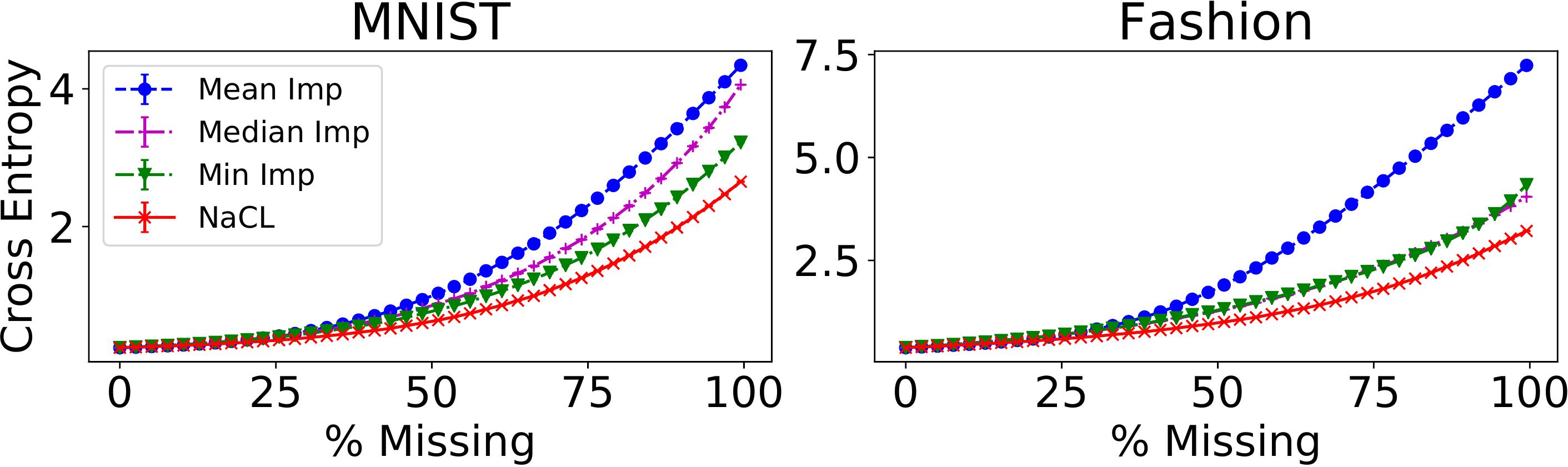}
            \caption{Cross Entropy (fidelity to logistic regression predictions)}
            \label{fig: mnist_cross_entropy}
        \end{subfigure}
        \begin{subfigure}[c]{.49\textwidth}
            \centering
            \includegraphics[height=0.28\columnwidth]{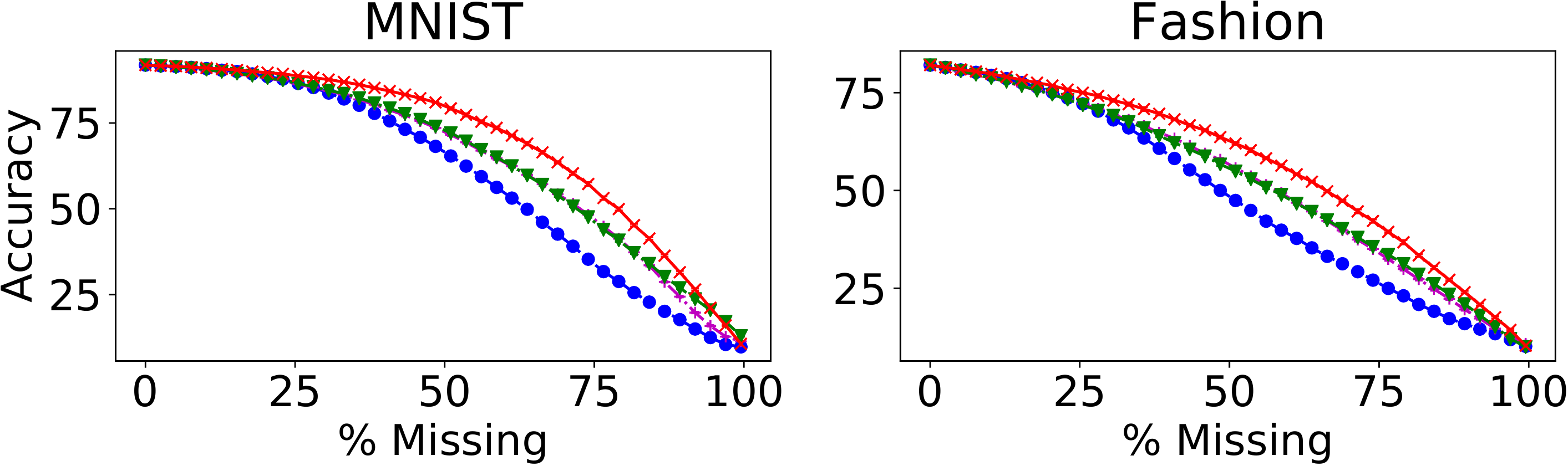}
            \caption{Accuracy (impact on test-set prediction quality)}
            \label{fig: mnist_accuracy}
        \end{subfigure}
        \caption{Standard image classification datasets: comparison of average cross entropy to the original predictions and classification accuracies between naive conformant learning (NaCL) and commonly used imputation methods. The conditional probabilities from NaCL are consistently closest to the full-data predictions, and NaCL consistently outperforms other methods with different percentages missing features. }
        \label{fig:acc-ce-combined}
\end{figure*}

\begin{samepage}
\begin{table*}[!ht]
    \centering
    {\fontsize{8.6}{9}\selectfont
    \begin{sc}
        \begin{tabular}{ @{}l c c c  c c c c c c c c c@{} }
        \toprule
        {Cross Entropy} & \multicolumn{4}{c}{CovType} & \multicolumn{4}{c}{Adult} & \multicolumn{4}{c}{Splice} \\
    	\cmidrule(lr){2-5} \cmidrule(lr){6-9} \cmidrule(lr){10-13}  
    	   {Under $\%$ Missing} &20$\%$ & 40$\%$ & 60$\%$ & 80$\%$ & 20$\%$ & 40$\%$ & 60$\%$ & 80$\%$ & 20$\%$ & 40$\%$ & 60$\%$ & 80$\%$\\
    	\midrule
Min Imputation    &   12.8 & 15.5 & 20.7 & 29.4 &                              41.0 & 49.2 & 55.4 & 59.6 &             71.3 & 81.8 & 97.2 & 117.2 \\
Max Imputation    &   52.6 & 89.1 & 133.5 & 187.7 &                            84.2 & 114.6 & 125.3 & 114.5 &             70.5 & 78.3 & 89.3 & 103.2 \\
Mean Imputation   &   12.8 & 15.6 & 21.2 & 30.5 &                              34.1 & 38.7 & 44.8 & 52.6 &             69.2 & \textbf{74.7} & \textbf{82.4} & 92.3 \\
Median Imputation &   12.8 & 15.7 & 21.4 & 30.8 &                              35.3 & 41.2 & 48.6 & 57.8 &             70.0 & 75.7 & 83.0 & \textbf{92.0} \\
\midrule \midrule
Naive Conformant Learning         &   \textbf{12.6} & \textbf{14.8} & \textbf{18.9} & \textbf{25.8} &   \textbf{33.6} & \textbf{37.0} & \textbf{41.2} & \textbf{46.6} &             \textbf{69.1} & \textbf{74.7} & 82.8 & 94.0 \\
        \bottomrule
        \end{tabular}
    \end{sc}
    }
    \caption{Three unbalanced UCI datasets with categorical features: comparison of average cross entropy to the original predictions between naive conformant learning (NaCL) and commonly used imputation methods. The closest are denoted in bold.}
    \label{table: cross entropy UCI datasets}
\end{table*}
\begin{table*}[!h]
    \centering
    {\fontsize{9}{9}\selectfont
    \begin{sc}
        \begin{tabular}{ @{}l c c c  c c c c c c c c c@{} }
        \toprule
        {Weighted F1} & \multicolumn{4}{c}{CovType} & \multicolumn{4}{c}{Adult} & \multicolumn{4}{c}{Splice} \\
    	\cmidrule(lr){2-5} \cmidrule(lr){6-9} \cmidrule(lr){10-13}  
    	 {under $\%$ Missing} &20$\%$ & 40$\%$ & 60$\%$ &80$\%$ & 20$\%$ & 40$\%$ & 60$\%$ & 80$\%$ & 20$\%$ & 40$\%$ & 60$\%$ & 80$\%$ \\
    	\midrule
Min Imputation    &            64.0 & 58.1 & 52.2 & 46.1 &                                       81.7 & 79.3 & 77.5 & \textbf{76.0} &                           86.9 & 69.8 & 49.2 & 38.8   \\    
Max Imputation    &            49.8 & 44.4 & 41.6 & 37.3 &                                       81.7 & 79.3 & 77.4 & \textbf{76.0} &                           86.9 & 69.8 & 49.1 & 38.8   \\
Mean Imputation   &            64.0 & 58.0 & 52.2 & 46.3 &                                       82.9 & 79.8 & 75.3 & 70.7 &                                    91.8 & 82.3 & 66.2 & 45.7   \\
Median Imputation &            64.0 & 58.1 & 52.2 & 46.1 &                                       82.7 & 79.2 & 74.8 & 70.5 &                                    89.4 & 77.6 & 59.5 & 42.5   \\
\midrule
\midrule
Naive Conformant Learning &    \textbf{66.1} & \textbf{61.7} & \textbf{56.9} & \textbf{51.7} &   \textbf{83.4} & \textbf{81.2} & \textbf{77.9} & 73.5 &         \textbf{93.3} & \textbf{87.2} & \textbf{76.6} & \textbf{59.1} \\
        \bottomrule
        \end{tabular}
    \end{sc}
    }
\caption{Three unbalanced UCI datasets with categorical features: comparison of weighted F1 scores between naive conformant learning (NaCL) and commonly used imputation. The highest are denoted in bold.}
\label{table: f1 scores uci dataset}
\end{table*}
\end{samepage}

\section{Empirical Evaluation}

In this section, we empirically evaluate the performance of naive conformant learning (NaCL) and provide a detailed discussion of our method's advantages over existing imputation approaches in practice.\footnote{Our implementation of the algorithm and experiments are available at \url{https://github.com/UCLA-StarAI/NaCL}.} More specifically, we want to answer the following questions:
\begin{description}
    \item[Q1] Does NaCL reliably estimate the probabilities of the original logistic regression with full data? How do these estimates compare to those from imputation techniques, including ones that also model a feature distribution?
    \item[Q2] Do higher-quality expectations of a logistic regression classifier result in higher accuracy on test data?
    \item[Q3] Does NaCL retain logistic regression's higher predictive accuracy over unconstrained naive Bayes?
\end{description}

\paragraph{Experimental Setup}
To demonstrate the generality of our method, we construct a 5-dataset testbed suite that covers assorted configurations \citep{yann2009mnist,fashion2017,blackard1999comparative,Dua2017,noordewier1991training}; see Table~\ref{table: datasets}. The suite ranges from image classification to DNA sequence recognition; from fully balanced labels to $>\!75\%$ of samples belonging to a single class; from continuous to categorical features with up to 40 different values. For datasets with no predefined test set, we construct one by a $80\!:\!20$ split. As our method assumes binary inputs, we transform categorical features through one-hot encodings and binarize continuous ones based on whether they are 0.05 standard deviation above their respective mean. 

Our algorithm takes as input a logistic regression model which we trained using fully observed training data. During prediction time, we make the features go missing uniformly at random based on a set missingness percentage, which corresponds to a missing completely at random (MCAR) mechanism~\citep{little2014statistical}. We repeat all experiments for 10 (resp.\ 100) runs on MNIST, Fashion, and CovType (resp.\ Adult and Splice) and report the average.


\subsection{Fidelity with the Original Predictions}
\label{sec: fidelity}
The optimal method to deal with missing values would be one that enables the original classifier to act as if no features were missing. 
In other words, we want the predictions to be affected as little as possible by the missingness. As such, we evaluate the similarity between predictions made with and without missingness, measured by the average cross entropy. The results are reported in Figure~\ref{fig: mnist_cross_entropy}\footnote{Max imputation results are dismissed as they are orders of magnitude worse than the rest.}
and Table~\ref{table: cross entropy UCI datasets}; the error bars were too small be visibly noticeable and were omitted in the plots. In general, our method outperforms all the baselines by a significant margin, demonstrating the superiority of the expected predictions produced by our method.

We also compare NaCL with two imputation methods that consider the feature distribution, namely EM~\citep{dempster1977maximum} and MICE~\citep{buuren2010mice}. EM imputation reports the second-to-worst average cross entropies and MICE's results are very similar to those of mean imputation when $1\%$ of features are missing. Due to the fact that both EM and MICE are excessively time-consuming to run and their imputed values are no better quality than more lightweight alternatives, we do not compare with them in the rest of the experiments. We would like to especially emphasize this comparison; it demonstrates that directly leveraging feature distributions without also considering how the imputed values impact the classifier may lead to unsatisfactory predictions, further justifying the need for solving the expected prediction task and conformant learning. This also concludes our answer to {Q1}.

\subsection{Classification Accuracy}
\label{section: accuracy}
Encouraged by the fact that NaCL produces more reliable estimates of the conditional probability of the original logistic regression, we further investigate how much it helps achieve better classification accuracy under different percentages of missing features (i.e., Q2). As suggested by Figure~\ref{fig: mnist_accuracy} and Table~\ref{table: f1 scores uci dataset},\footnote{We report weighted F1 scores as the datasets are unbalanced.} NaCL consistently outperforms all other methods except on the Adult dataset with $80\%$ of the features missing.

Lastly, to answer {Q3} we compare NaCL to a maximum-likelihood naive Bayes model.\footnote{We do not report the full set of results in the table because maximum-likelihood learning of naive Bayes optimizes for a different loss and effectively solves a different task than NaCL and the imputation methods.} In all datasets except Splice, logistic regression achieves higher classification accuracy than naive Bayes with fully observed features. NaCL maintains this advantage until $40\%$ of the features go missing, further demonstrating the effectiveness of our method. Note that these four datasets have a large number of samples, which is consistent with the prevailing consensus that discriminative learners are better classifiers given a sufficiently large number of samples~\citep{ng2002discriminative}.

\section{Case Study: Sufficient Explanations}

In this section we briefly discuss utilizing conformant learning to explain classifications and show some empirical examples as a proof of concept.

On a high level, the task of explaining a particular classification can be thought of as quantifying the ``importance'' of each feature and choosing a small subset of the most important features as the explanation. Linear models are widely considered easy to interpret, and thus many explanation methods learn a linear model that is closely faithful to the original one, and then use the learned model to assign importance to features \citep{lime,NIPS2017_7062,pmlr-v70-shrikumar17a}. 
These methods often assume a black-box setting, and to generate explanations they internally evaluate the predictor on multiple perturbations of the given instance. A caveat is that the perturbed values may have a very low probability on the distribution the classifier was trained on. This can lead to unexpected results as machine learning models typically only guarantee generalization if both train and test data are drawn from the same distribution. 

Instead we propose to leverage the feature distribution in producing explanations. To explain a given binary classifier, we consider a small subset of feature observations that is sufficient to get the same classification, in expectation w.r.t.\ a feature distribution. Next, we formally define our method:
\begin{definition}{(Support and Opposing Features)}
    Given $\function$, $P$, and $\xs$, we partition the given feature observations into two sets. The first set consists of the \emph{support} features that contribute towards the classification of $\function(\xs)$:
    \begin{align*}
        \xs_{+} = \begin{cases}
            \left\{ x \in \xs:  E_{\function, P}(\xs \setminus x) \leq \function(\xs) \right\}& \text{ if $\function(\xs) \geq 0.5$, } \\
            \left\{ x \in \xs:  E_{\function, P}(\xs \setminus x) > \function(\xs) \right\}& \text{ otherwise. }
        \end{cases}
    \end{align*}
    The rest are the \emph{opposing} features that provide evidence against the current classification: $\xs_{-} = \xs \setminus \xs_{+}$.
\end{definition}

\begin{definition}{}
    \emph{Sufficient explanation} of $\function(\xs)$ with respect to $P$ is defined as the following:
    \begin{align*}
        &\argmin_{\es \subseteq \xs_{+}}\abs{\es} \\
        \text{\:s.t.\:} \sgn(E_{\function,P}(\es\xs_{-}) &- 0.5) = \sgn(\function(\xs) - 0.5)
    \end{align*}
\end{definition}
Intuitively, this is the smallest set of support features that, in expectation, result in the same classification despite all the evidence to the contrary. In other words, we explain a classification using the strongest evidence towards it.

\begin{figure}[tb]
    \centering
    \begin{subfigure}[c]{.9\columnwidth}
        \centering
        \includegraphics[width=0.9\columnwidth]{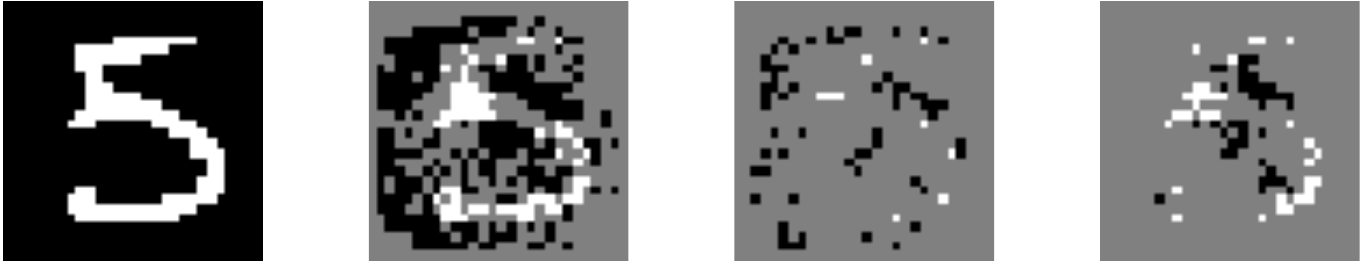}
        \includegraphics[width=0.9\columnwidth]{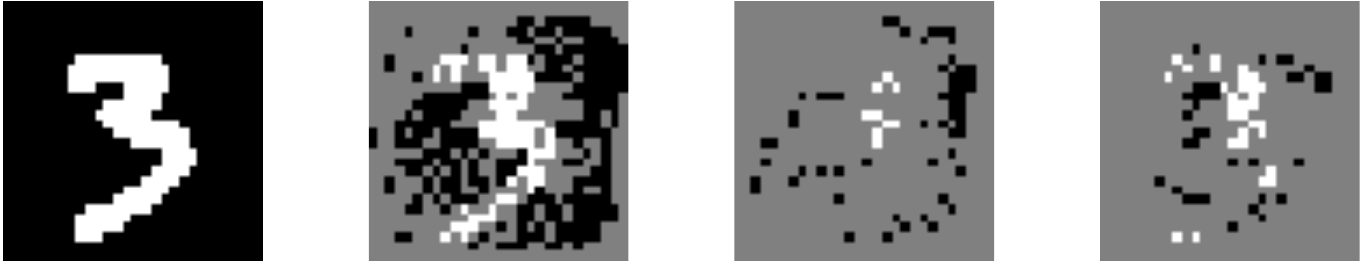}
        \caption{Correctly classified examples}
        \label{fig:suf_correct_classified}
    \end{subfigure}
    \begin{subfigure}[c]{.9\columnwidth}
        \centering
        \includegraphics[width=0.9\columnwidth]{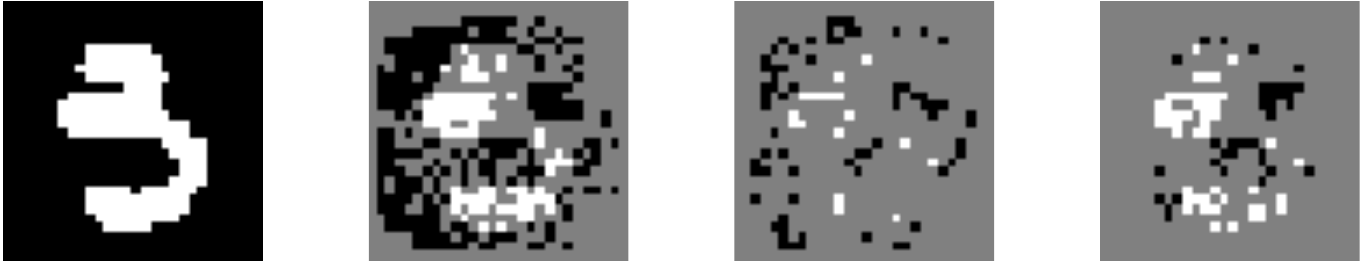}
        \includegraphics[width=0.9\columnwidth]{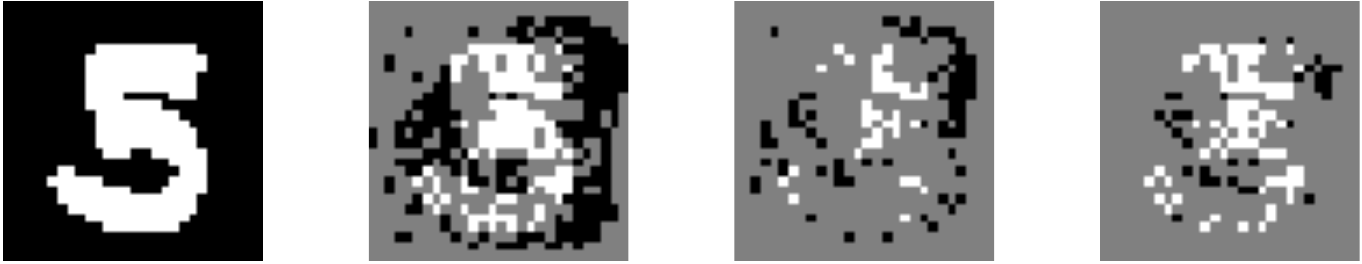}
        \caption{Misclassified examples}
        \label{fig:suf_miss_classified}
    \end{subfigure}
  \caption{
    Explanations for MNIST classifications. Grey features were not chosen as explanations; white/black are the true color of chosen features.
    From left to right:
    1) all features;
    2) support features;
    3) top-$k$ support features;
    4) sufficient explanation of size $k$.
  }
  \label{fig:explain}
\end{figure}

For a qualitative evaluation, we generate sufficient explanations on instances of a binary logistic regression task for MNIST digits 5 and 3; see the last column of Figure~\ref{fig:explain}.
Take the first example in Figure~\ref{fig:suf_correct_classified}: the white pixels selected as sufficient explanation show that the digit should be a 5. Also notice the black pixels in the explanation: they express how the absence of white pixels significantly contributes to the classification, especially in parts of the image where they would be expected for the opposing class.
Similarly, the black pixels in the first example in Figure~\ref{fig:suf_miss_classified} look like a 3, and the white pixels in the explanation look like a 5, explaining why this 3 was misclassified as a 5.
We further compare our approach to an alternative one that selects a subset of support features based on their logistic regression weights; see the third column of Figure~\ref{fig:explain}. It selects features that will cause a large difference in prediction if the value was flipped, as opposed to missing, which is what sufficient explanation considers.

\section{Related Work}
There have been many approaches developed to classify with missing values, which can broadly be grouped into two different types. The first one focuses on increasing classifiers' inherent robustness to feature corruption, which includes missingness. A common way to achieve such robustness is to spread the importance weights more evenly among features \citep{Globerson:2006:NTT:1143844.1143889,dekel2008learning,xia2017adjusted}. One downside of this approach is that the trained classifier may not achieve its best possible performance if no features go missing. 

The second one investigates how to impute the missing values. In essence, imputation is a form of reasoning about missing values from observed ones \citep{sharpe1995dealing,batista2002study,mcknight2007missing}. An iterative process is commonly used during this reasoning process \citep{buuren2010mice}. Some recent works also adapt auto-encoders and GANs for the task \citep{costa2018missing,mattei2019miwae}. Some of these works can be incorporated into a framework called multiple imputations to reflect and better bound one's uncertainty \citep{schafer1999multiple,azur2011multiple,gondara2018mida}. These existing methods focus on substituting missing values with those closer to the ground truth, but do not model how the imputed values interact with the trained classifier. On the other hand, our proposed method explicitly reasons about what the classifier is expected to return. 

We are among the first to incorporate feature distributions to generate explanations. Notable recent work along this line includes \citet{pmlr-v80-chen18j}, which proposes to maximize the mutual information between selected features and the class. To more explicitly leverage a feature distribution, \citet{chang2018explaining} proposes to explain a classification by a subset of features that maximally affect the classifier output, when its values are substituted by in-fills sampled from the feature distribution conditioned on the rest of the features.
This contrasts with our method which studies the affect of certain features on a classifier by marginalizing, rather than sampling.

\section{Conclusion \& Future Work}

In this paper we introduced the expected prediction task, a principled approach to predicting with missing features. It leverages a feature distribution to reason about what a classifier is expected to return if it could observe all features. We then proposed conformant learning to learn joint distributions that conform with and can take expectations of discriminative classifiers. A special instance of it--naive conformant learning--was shown empirically to outperform many existing imputation methods.
For future work, we would like to explore conformant learning for other generative-discriminative pairs of models, and extend NaCL to real-valued features.


\section*{Acknowledgements}

This work is partially supported by NSF grants \#IIS-1657613, \#IIS-1633857, \#CCF-1837129, DARPA
XAI grant \#N66001-17-2-4032, NEC Research, and gifts from Intel and Facebook Research.

\appendix
\setcounter{equation}{0}

\section{Proofs}\label{sec:appx-proofs}

\subsection{Proof of Theorem~\ref{thm:lr-over-nb}}
    The proof is by reduction from computing the same-decision probability, whose decision problem $\textbf{D-SDP}$ was shown to be \NP-hard.~\citep{chen2013sdp}

    Given a naive Bayes distribution $P(.)$ over variables $C$ and $\Xs$, a threshold $T$, and a probability $p$, $\textbf{D-SDP}$ asks: is the same-decision probability $\sum_{\xs} \mathbb{I}(P(c \given \xs) > T) P(\xs)$
    greater than $p$?
    Here, $\mathbb{I}(.)$ denotes an indicator function which returns 1 if the enclosed expression is true, and 0 otherwise.

    Using Lemma~\ref{lem:nb-to-lr} we can efficiently translate a naive Bayes model $P$ to a logistic regression with a weight function $w(.)$ such that
    \begin{equation*}
        P(c \given \xs) = \frac{1}{1+e^{-w(\xs)}}.
    \end{equation*}
    Note that $P(c \given \xs) > T$ iff $w(\xs) > -\log(\frac{1}{T} - 1)$. Then we construct another logistic regression with weight function
    \begin{equation*}
        w^\prime(\xs) = n \cdot \left(w(\xs) + \log\left(\frac{1}{T}-1\right)\right),
    \end{equation*}
    for some positive constant $n$. As $w$ is a linear model, $w^\prime$ is also linear, and $w^\prime(\xs) > 0$ iff $P(c \given \xs)\!>\!T$.
    As $n$ grows, $w^\prime(.)$ approaches $\infty$ and $-\infty$ for positive and negative examples, respectively. Hence, this logistic regression model outputs 1 if $P(c\given\xs)\!>\!T$ and 0 otherwise, effectively being an indicator function.
    Therefore, the expectation of such classifier over $P(\Xs)$ is equal to the same-decision probability of $\Xs$.
\qed

\subsection{Proof of Lemma \ref{lem:nb-to-lr}}

We want to prove there is a unique $\function$ such that $\function(\xs) = {P(c \mid \xs)}$ for all $\xs$ given naive Bayes distribution $P$.
Using Bayes' rule and algebraic manipulation, we get:
\begin{align}
    &P (c \mid \xs) 
    =  \frac{ P(\xs \mid c)\ P(c)} {P(\xs \mid c)\ P(c) + P(\xs \mid \notc)\ P(\notc)} \nonumber \\ 
    =& \cfrac{1}{1 + \cfrac{ P(\xs \mid \notc)\ P(\notc)}{P(\xs \mid c)\ P(c)}}
    = \cfrac{1}{1 + \exp \bigg[-\log \cfrac{ P(\xs \mid c)\ P(c)}{P(\xs \mid \notc)\ P(\notc)} \bigg] }  \nonumber 
\end{align}
For any input $\xs$, we want above quantity to be equal to $\function(\xs) = 1 / (1 + \exp[-\sum_i w_i x_i])$.
In other words, we need:
\begin{equation*}
    \log \cfrac{ P(\xs \mid c)\ P(c)}{P(\xs \mid \notc)\ P(\notc)}
    = 
    \sum_i w_i\ x_i
\end{equation*}

\noindent Using naive Bayes assumption, we arrive at:
\begin{equation}\label{eq:lemma1_clnb}
    \sum_{i=0}^{n} w_i x_i = \log \frac{P(c)}{P(\notc)} + \sum_{i=1}^n \log \frac{P(x_i \mid c)}{P(x_i \mid \notc)}
\end{equation}

Now we want the RHS of Equation~\ref{eq:lemma1_clnb} to be a linear function of $x_i$'s, so we do the following substitution for $i > 0$ assuming binary features:
\begin{align}
    \label{eq:clnb_sub} \log \cfrac{P(x_i \mid c)}{P(x_i \mid \notc)} =&\, (x_i) \cdot \log \cfrac{P(x_i = 1 \mid c)}{P(x_i = 1 \mid \notc)}\  \\
    \nonumber &+ (1 - x_i) \cdot \log \cfrac{P(x_i = 0 \mid c)}{P(x_i = 0 \mid \notc)} 
\end{align}

By combining Equations~\ref{eq:lemma1_clnb} and \ref{eq:clnb_sub} we get the weights in Lemma~\ref{lem:nb-to-lr} by simple algebraic manipulations. To solve for the bias term $w_0$ we plug in $x_i\!=\!0$ for all $i\!>\!0$. To compute $w_i$ for a non-zero $i$ we take the coefficient of $x_i$ in Equation~\ref{eq:clnb_sub}.
%
%
\qed

\subsection{Proof of Lemma \ref{lem:lr-to-nb}}
Through the same algebraic manipulation as before, we get the same equations as in Lemma~\ref{lem:nb-to-lr} with the only difference being that we are now solving the parameters of a naive Bayes model rather than weights of the logistic regression model. Intuitively, because the NB model has $2n+1$ free parameters but the LR model only has $n+1$ parameters, we expect some degree of freedom. To get rid of the freedom and get a unique solution, we fix the values for $n$ parameters as follows:
\begin{equation}
  P(x_i = 1 \mid c) = \theta_i \label{eq:lemma2_fixparameters}
\end{equation}
Without loss of generality we have fixed the parameter values for positive features. One can equally set the values in other ways as long as one parameter value per feature is fixed. 

Now there is a unique naive Bayes model that matches the logistic regression classifier and also agrees with Equation~\ref{eq:lemma2_fixparameters}. That is, the remaining $n+1$ parameter values are given by the LR parameters, resulting in the following parameters for such naive Bayes model:
\begin{alignat*}{2}
    &P(x_i \mid c) = \theta_i, 
    &&P(\bar{x}_i \mid c) = 1-P(x_i \mid c), \\
    &P(x_i \mid \notc) = \cfrac{1}{1 + e^{w_i} \frac{1-\theta_i}{\theta_i}}, 
    &&P(\bar{x}_i \mid \notc) = 1-P(x_i \mid \notc), \\
    &P(c) = \sigmoid\bigg(w_0 -  \sum_{i=1}^{n} \log &&\cfrac{ P(\bar{x}_i \mid c)}{ P(\bar{x}_i \mid \notc)}\bigg).
\end{alignat*}
\qed

\section{Beyond Binary Classification: Multiclass}\label{sec:appx-multiclass}

In the paper, we studied logistic regression and conformant naive Bayes models assuming binary classification. We now show that our method can easily be extended to multiclass. We first modify our notation of logistic regression and naive Bayes models to allow for an arbitrary number of classes.
\begin{definition}{(Multiclass Classifiers)} \label{def:multiclass}
    Suppose we have a classifier with $K$ classes, each denoted by $c_k$ ($k \in [0, K-1]$). Then $\function_k$ denotes the conditional probability for class $c_k$ in logistic regression, and $P$ a naive Bayes distribution defined as:
\begin{gather*}
    \function_k(\xs) = \cfrac{e^{W_k \cdot \xs}}{\sum_j e^{W_j \cdot \xs}} \\
    P(c_k \mid \xs) = \cfrac{ P(c_k) \prod_{i=1}^{n} P(x_i \mid c_k) }{\sum_j P(c_j) \prod_{i=1}^{n} P(x_i \mid c_j) }
\end{gather*}
\end{definition}
\noindent We say $P$ conforms with $\function$ if their predictions agree for all classes: $P(c_k \given \xs) = \function_k(\xs)$ for all $k$ and $\xs$.

Next, we describe naive conformant learning for multiclass.
Instead of directly matching the predictions of $P$ and $\function$ for all classes, we match their ratios in order to simplify equations going forward.
Moreover, to get the same classifiers it suffices to divide by the probability of only one class, so without loss of generality we set the following be true.
\begin{equation*}
     \cfrac{\function_k(\xs)}{ \function_0(\xs) } = \cfrac{ P(c_k \mid \xs)}{ P(c_0 \mid \xs) }, \quad \forall\: k \in [1,K-1]
\end{equation*}
Using Definition~\ref{def:multiclass}, this leads to
\begin{equation*}
     \cfrac{e^{-W_k \cdot \xs}}
          {e^{-W_0 \cdot \xs}} = 
     \cfrac{P(c_k) \prod_{i=1}^{n} P(x_i \mid c_k)  }
          {P(c_0) \prod_{i=1}^{n} P(x_i \mid c_0)  },
    \quad \forall\: k \in [1,K-1].
\end{equation*}


The parameters of a multiclass naive Bayes are: $\pa{c_k}{}\!=\!P(c_k)$, $\pa{x_i}{c_k}\!=\!P(x_i = 1 \mid c_k)$, and $\pa{\bar{x}_i}{c_k}\!=\!P(x_i = 0 \mid c_k)$. Then, we get the following constraints for NaCL through similar algebraic manipulations as in the binary case:
\begin{gather}
  \label{eq:constraint_multi_pk} \sum\nolimits_k \pa{c_k}{} = 1  \\
    \label{eq:constraint_alpha_i} \pa{x_i}{c_k} + \pa{\bar{x}_i}{c_k}  = 1, \quad\forall\:i,k\!>\! 0  \\
    e^{w_{k,i} - w_{0, i}}\ \ \pa{x_i}{c_k}^{-1}\ \ \pa{\bar{x}_i}{c_k} \ \ \pa{x_i}{c_0}\ \ \pa{\bar{x}_i}{c_0}^{-1} = 1, \quad\forall\:i,k\!>\!0  \nonumber \\  
    e^{w_{k, 0} - w_{0,0}}\ \ \pa{c_k}{}^{-1}\ \ \pa{c_0}{}  \prod_{i=1}^{n} \pa{\bar{x}_i}{c_k}^{-1} \ \ \pa{\bar{x}_i}{c_0} = 1, \quad\forall\:k\!>\!0 \nonumber
\end{gather}
Again, we relax Equations~\ref{eq:constraint_multi_pk} and \ref{eq:constraint_alpha_i} to inequalities to obtain valid geometric program constraints. The rest of the method stays the same: we maximize the marginal likelihood with above constraints by minimizing the inverse of the joint likelihood on a completed datset, as described in Section~\ref{sec:logistic_regression_expectations}.

\bibliographystyle{named}
\bibliography{classifier_expect_ijcai}

\end{document}